\documentclass{article}

\usepackage[preprint]{neurips_2026}

\usepackage[utf8]{inputenc} 
\usepackage[T1]{fontenc}    
\usepackage{hyperref}       
\usepackage{url}            
\usepackage{booktabs}       
\usepackage{amsfonts}       
\usepackage{nicefrac}       
\usepackage{microtype}      
\usepackage{xcolor}         
\usepackage[utf8]{inputenc}
\usepackage[T1]{fontenc}
\usepackage{hyperref}
\usepackage{url}
\usepackage{booktabs}
\usepackage{amsfonts}
\usepackage{amsmath}
\usepackage{amssymb}
\usepackage{amsthm}
\usepackage{mathtools}
\usepackage{nicefrac}
\usepackage{microtype}
\usepackage{xcolor}
\usepackage{graphicx}
\usepackage{caption}
\usepackage{subcaption}
\usepackage{enumitem}
\usepackage{bbm}
\usepackage{stmaryrd}
\usepackage{stfloats}
\usepackage{wrapfig}
\newtheorem{theorem}{Theorem}[section]
\newtheorem{proposition}[theorem]{Proposition}

\newtheorem{definition}[theorem]{Definition}

\DeclareMathOperator{\rk}{rk}

\usepackage{hyperref}
\usepackage{stfloats}

\title{TopoU-Net: A U-Net Architecture for Topological Domains}

\author{%
  Gaurav Gaurav\\
  University of South Florida\\
  \And
  Ibrahem ALJabea\\
  Louisiana State University\\
  \And
  Yaroslav Zakomornyy\\
  University of South Florida\\
  \And
  Eric Frank\\
  Vinci4D\\
  \And
  Mohamed Elhamdadi\\
  University of South Florida\\
  \And
  Theodore Papamarkou\\
  PolyShape\\
  NTUA
  \And
  Mustafa Hajij\\
  USFCA
}

\begin{document}

\maketitle

\begin{abstract}
Many modern datasets mix points, edges, regions, groups, objects, events, hyperedges, and relations. Yet neural architectures often force such data into grids, graphs, or sequences, obscuring higher-order structure and making encoder-decoder designs domain-specific. We view U-Net not as a grid-specific architecture, but as a hierarchical encoder-decoder principle: representation spaces, transport maps between levels, and skip connections between matched levels. Combinatorial complexes naturally supply these ingredients through cells, incidences, and ranks. We introduce TopoU-Net, a rank-path U-Net for topological domains. Given a path from an input rank to a bottleneck rank and back, the encoder lifts cochains upward along incidence maps, the decoder transports them downward, and skip connections merge features at matched ranks. Rank replaces spatial scale: choosing paths through nodes, edges, faces, hyperedges, or global cells becomes the central architectural decision. A key quantity is the bottleneck support ratio, the number of cells at the bottleneck relative to the number of cells at the input rank. This ratio is fixed by the complex and chosen path rather than by arbitrary pooling, and it clarifies when skip connections are optional, useful, or structurally important. Across node classification, graph classification, hypergraph node classification, mesh classification, and image reconstruction, TopoU-Net provides a reusable encoder-decoder template for higher-order structured data. Among the evaluated baselines, it achieves the strongest mean accuracy on six of eight node-classification datasets and four of five hypergraph datasets, with the largest gains on heterophilic graphs. Ablations show that removing skip connections is most damaging under severe bottleneck compression.
\end{abstract}

\section{Introduction}

Deep learning architectures are most effective when their inductive bias
matches the structure of the data. CNNs assume signals live on grids, GNNs
assume pairwise relational structure, and Transformers assume collections of
tokens with learned interactions. Many modern datasets, however, are not
naturally described by a single one of these abstractions. They contain
features attached to points, relations attached to edges, regions attached to
faces, groups attached to hyperedges, and sometimes global or higher-order
objects. Forcing such data into only a grid, graph, or sequence can discard
the very structure the model should exploit.

This raises a basic architectural question: how should information move
between the different supports on which data lives? In images, U-Net
architectures~\cite{ronneberger2015unet} answer this question through a
spatial hierarchy: an encoder compresses local features into coarser context,
a decoder returns this context to the output resolution, and skip connections
restore fine-scale information lost through compression. The same principle,
however, is not inherently grid-specific. What it requires is a hierarchy of
representation spaces, transport maps between them, and skip connections
between matched levels.

In non-Euclidean domains, the central question is therefore what should play
the role of scale. Graph U-Nets \cite{gao2019graph} answer this through learned pooling and
unpooling over nodes, but they remain tied to pairwise node--edge structure.
Many scientific and relational domains contain higher-order objects---faces
in meshes, hyperedges, triangles and cliques in networks, or patches in
images---where hierarchy is better understood as traversal across
topological ranks rather than merely reduction in node count. This is the
view taken in this paper: combinatorial complexes supply the representation
spaces through cells, the transport maps through incidences, and the hierarchy
through rank.

This paper introduces \emph{TopoU-Net}, a U-Net architecture on combinatorial complexes~\cite{hajij2023tdl}. 
A combinatorial complex \(\mathcal{X}\) organizes the domain into ranked cells (nodes, edges, triangles, hyperedges, or higher-order cells). 
TopoU-Net selects an encoder rank path \(\mathcal{S}=(s_0<\cdots<s_L)\) and transports features upward along incidence maps; the decoder reverses the path, while skip connections merge encoder and decoder features at matched ranks. 
Thus scale is replaced by rank traversal, turning the architecture into a U-Net over the cell structure of the domain rather than a Euclidean grid.

The rank path is the central design choice: edges capture pairwise relations, triangles or hyperedges incorporate higher-order interactions, and a global cell yields strong compression. 
This makes compression explicit via the \emph{bottleneck support ratio} \(\rho_{\mathrm{bot}}=n_{s_L}/n_{s_0}\), fixed by the complex and chosen path. 
When \(\rho_{\mathrm{bot}}\ll 1\),
 reconstruction from the bottleneck is structurally constrained; skip connections provide a principled way to bypass severe compression, offering a concrete design principle for topological encoder--decoder networks.

TopoU-Net gives a unified encoder--decoder view across domains. Images use
pixels, adjacencies, and patches; graphs use nodes, edges, triangles, and
global components; hypergraphs use node--hyperedge incidence; and meshes use
vertices, edges, and faces. In all cases, the architecture is unchanged: only
the complex construction and rank path vary. We evaluate TopoU-Net on node classification, graph classification,
hypergraph node classification, mesh classification, and image reconstruction.

Across the evaluated baselines, TopoU-Net achieves the strongest mean accuracy
on six of eight node-classification datasets and four of five hypergraph
datasets, with the largest gains on heterophilic graphs where edge-local
aggregation is unreliable. Reconstruction experiments show that
incidence-based rank transport can be parameter-efficient on grids, while
skip-connection ablations support the predicted link between bottleneck
support ratio and skip importance.

\textbf{Contributions.}
We make four contributions. First, we formulate U-Net-style encoder--decoders
on combinatorial complexes using incidence-based transport across ranks.
Second, we identify the bottleneck support ratio as a structural measure of
rank-path compression, explaining when skips or wider bottlenecks are needed.
Third, we give a canonical TopoU-Net with same-rank refinement and matched-rank
skip merges. Fourth, we evaluate it across graph, hypergraph, mesh, and image
domains, with ablations isolating the roles of rank paths and skips.

\section{Related work}
\label{sec:related-work}

\textbf{U-Net architectures.}
U-Net was introduced as an encoder-decoder architecture with skip connections
for biomedical image segmentation~\cite{ronneberger2015unet}. Its main
architectural bias is to combine coarse contextual features with fine-scale
features preserved through matched skip connections. Many variants refine this
principle. Attention U-Net gates skip features using task-dependent attention
signals~\cite{oktay2018attention}, U-Net++ introduces nested skip pathways to
improve feature fusion~\cite{zhou2018unetplusplus}, ResUNet adds residual
blocks for deeper encoder-decoder models~\cite{zhang2018resunet}, and UNet3+
uses full-scale skip aggregation across resolutions~\cite{huang2020unet3plus}.
These architectures are defined on Euclidean grids, where locality and scale
are provided by convolution, pooling, and upsampling. TopoU-Net keeps the
encoder-decoder and skip-connection principle, but replaces Euclidean scale by
rank structure in a combinatorial complex.

\textbf{Graph neural networks and graph hierarchy.}
Graph neural networks learn representations by aggregating information over
pairwise graph neighborhoods. Representative architectures include
GCN~\cite{kipf2017semi}, GraphSAGE~\cite{hamilton2017inductive},
GAT~\cite{velickovic2018graph}, and GIN~\cite{xu2019powerful}. Extensions such
as MixHop~\cite{abu2019mixhop} and H2GCN~\cite{zhu2020beyond} incorporate
multi-hop or heterophily-aware aggregation. Hierarchical graph models introduce
coarser graph representations through learned assignment or node selection,
including DiffPool~\cite{ying2018diffpool}, Graph U-Net~\cite{gao2019graph},
top-$k$ pooling~\cite{cangea2018towards}, and MinCutPool~\cite{bianchi2020spectral}.
These methods define hierarchy inside a pairwise graph domain. By contrast,
TopoU-Net defines hierarchy by transporting features across ranked cells, such
as nodes, edges, triangles, hyperedges, or higher-order cells.

\textbf{Topological deep learning.}
Topological deep learning (TDL) extends neural architectures from graphs to richer
domains such as simplicial complexes, cell complexes, hypergraphs, and
combinatorial complexes~\cite{hajij2023tdl,hajij2024topox,zia2024tdlreview}.
Hodge-theoretic methods generalize graph Laplacians to signals on higher-order
cells~\cite{lim2020hodge,schaub2021signal,roddenberry2019hodgenet,barbarossa2020topological},
while message-passing and attention architectures aggregate over upper, lower,
and boundary neighborhoods~\cite{hajij2020cell,bodnar2022b,giusti2022simplicial,giusti2023cell,battiloro2023generalized}.
These models provide higher-order analogues of graph neural networks, but they
are usually not organized as U-shaped encoder-decoder architectures with
matched cross-rank skip connections.

\textbf{Topological pooling and our distinction.}
Pooling and coarsening on higher-order domains are less standardized than graph
pooling. Existing simplicial pooling methods adapt graph-pooling ideas to
simplicial complexes~\cite{cinque2023pooling}, and skip connections have been
studied for higher-order networks~\cite{hajij2022highskip}. TopoU-Net takes a
different approach. Instead of learning a new coarsened complex at each layer,
TopoU-Net uses the existing rank structure of a combinatorial complex as the
hierarchy. Incidence maps transport features across ranks, within-rank
operators refine features at fixed support, and skip connections merge encoder
and decoder states at matched ranks. This uses the compression induced by the
rank path and links architectural choices to the cell counts of the
underlying domain.

\section{TopoU-Net: architecture and compression}
\label{sec:topounet}

TopoU-Net is an encoder--decoder architecture on combinatorial complexes~(CCs). 
Ranked cells supply the hierarchy, incidence matrices provide transport between ranks, and skip connections merge encoder and decoder features at matched ranks. 
We first fix notation for cells, cochains, and incidence-based transport; the architecture is defined in Section~\ref{subsec:topounet-definition}.

\subsection{Combinatorial complexes and cochain spaces}
\label{subsec:cc-preliminaries}

We use CC as the underlying topological domain. They allow
signals to be placed on vertices, edges, faces, hyperedges, or other higher-order cells.

\begin{definition}
A \emph{combinatorial complex (CC)} is a triple
$(S,\mathcal X,\rk)$, where $S$ is a finite nonempty vertex set,
$\mathcal X \subseteq 2^S\setminus\{\emptyset\}$ is a finite collection of
nonempty subsets containing every singleton $\{v\}$ for $v\in S$, and
$\rk:\mathcal X\to \mathbb Z_{\geq 0}$ is a rank function satisfying
$\rk(\{v\})=0$ for all $v\in S$, and
$x\subseteq y \implies \rk(x)\leq \rk(y)$.
\end{definition}

Elements of $\mathcal X$ are called \emph{cells}. A cell $x\in\mathcal X$
with $\rk(x)=r$ is called an $r$-cell. We write
$\mathcal X^r := \{x\in\mathcal X:\rk(x)=r\}$ and
$n_r := |\mathcal X^r|$, and define the active ranks by
$\mathcal R(\mathcal X) := \{r \in \mathbb Z_{\geq 0}:\mathcal X^r\neq\emptyset\}$.
The dimension of the complex $\mathcal X$ is
$\dim(\mathcal X):=\max \{\rk(x)\}_{ x\in\mathcal X}$.
Unlike simplicial complexes, a CC does not need to contain all
subsets of a cell, and the active ranks do not need to be consecutive. This is useful
for learning because the model can use only the ranks that are present or
useful for a given task.

\begin{wrapfigure}[8]{r}{0.26\textwidth}
\vspace{-\baselineskip}
\centering
\includegraphics[width=0.28\textwidth]{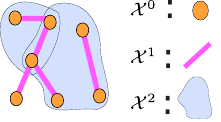}
\caption{A combinatorial complex with cells at ranks $0,1,$ and $2$.}
\label{fig:cc}
\end{wrapfigure}

For an active rank $r$, a feature signal on the $r$-cells is represented as
an $r$-cochain. We identify the corresponding cochain space with
$C^r(\mathcal X;\mathbb R^{d_r}) := \mathbb R^{n_r\times d_r}$,
where rows are indexed by cells in $\mathcal X^r$, and $d_r$ is the feature
dimension at rank $r$. Thus
$H^{(r)}\in C^r(\mathcal X;\mathbb R^{d_r})$ assigns one feature vector to
each $r$-cell.

\subsection{Incidence operators and rank transport}
\label{subsec:rank-transport}

Fix an ordering of the cells at each active rank. For two active ranks
$r<r'$, the incidence matrix
$B_{r,r'}\in\{0,1\}^{n_r\times n_{r'}}$ is defined by
$[B_{r,r'}]_{ij}=1$ if and only if $x_i^r \subset x_j^{r'}$, where
$x_i^r\in\mathcal X^r$ and $x_j^{r'}\in\mathcal X^{r'}$. Multiplying
by $B_{r,r'}^\top$ moves features upward from rank $r$ to rank $r'$, while
multiplying by $B_{r,r'}$ moves features downward from rank $r'$ to rank
$r$: $B_{r,r'}^\top H^{(r)}\in \mathbb R^{n_{r'}\times d_r}$ and
$B_{r,r'} H^{(r')}\in \mathbb R^{n_r\times d_{r'}}$.
Because $B_{r,r'}$ is defined directly from inclusion, transport can be
defined between any two active ranks $r<r'$.

Incidence also induces within-rank adjacency. For example,
$A_{r\mid r'} := B_{r,r'}B_{r,r'}^\top\in \mathbb R^{n_r\times n_r}$
connects two $r$-cells when they are incident to a common $r'$-cell. In
TopoU-Net, we distinguish these two roles; incidence matrices transport
features across ranks, while adjacency-type operators such as $A_{r\mid r'}$
may be used to refine features within a fixed rank.

\begin{definition}[Rank transport]
\label{def:rank-transport}
Let $r,r'\in\mathcal R(\mathcal X)$ with $r<r'$. An \emph{upward rank
transport} is a map
$T^{\uparrow}_{r\to r'}:
C^r(\mathcal X;\mathbb R^{d_r})
\to
C^{r'}(\mathcal X;\mathbb R^{d_{r'}})$
constructed from the incidence relation between $r$-cells and $r'$-cells.
The corresponding \emph{downward rank transport} is a map
$T^{\downarrow}_{r'\to r}:
C^{r'}(\mathcal X;\mathbb R^{d_{r'}})
\to
C^r(\mathcal X;\mathbb R^{d_r})$
constructed by reversing the incidence direction.
\end{definition}

The simplest transport pair is incidence convolution:
$T^{\uparrow}_{r\to r'}(H)
=
\phi\!(B_{r,r'}^\top H W^{\uparrow}_{r,r'})$, with
$W^{\uparrow}_{r,r'}\in\mathbb R^{d_r\times d_{r'}}$, and
$T^{\downarrow}_{r'\to r}(G)
=
\phi\!(B_{r,r'} G W^{\downarrow}_{r',r})$, with
$W^{\downarrow}_{r',r}\in\mathbb R^{d_{r'}\times d_r}$.
Here $\phi$ is a pointwise nonlinearity. In implementations, $B_{r,r'}$
may be replaced by a fixed degree-normalized incidence matrix. This changes
the weighting of incident cells, but not the type of transport map.

Throughout the paper, an \emph{encoder rank path} is an increasing sequence of
active ranks
$\mathcal S=(s_0<s_1<\ldots<s_L)\subseteq\mathcal R(\mathcal X)$.
The bottleneck rank is $s_L$. The corresponding U-shaped traversal is
$\mathcal P=(s_0,s_1,\ldots,s_L,s_{L-1},\ldots,s_0)$.
$\mathcal S$ determines the encoder
compression structure, while $\mathcal P$ describes the full
encoder--decoder path.

\begin{figure}[t]
    \centering
    \includegraphics[width=\textwidth]{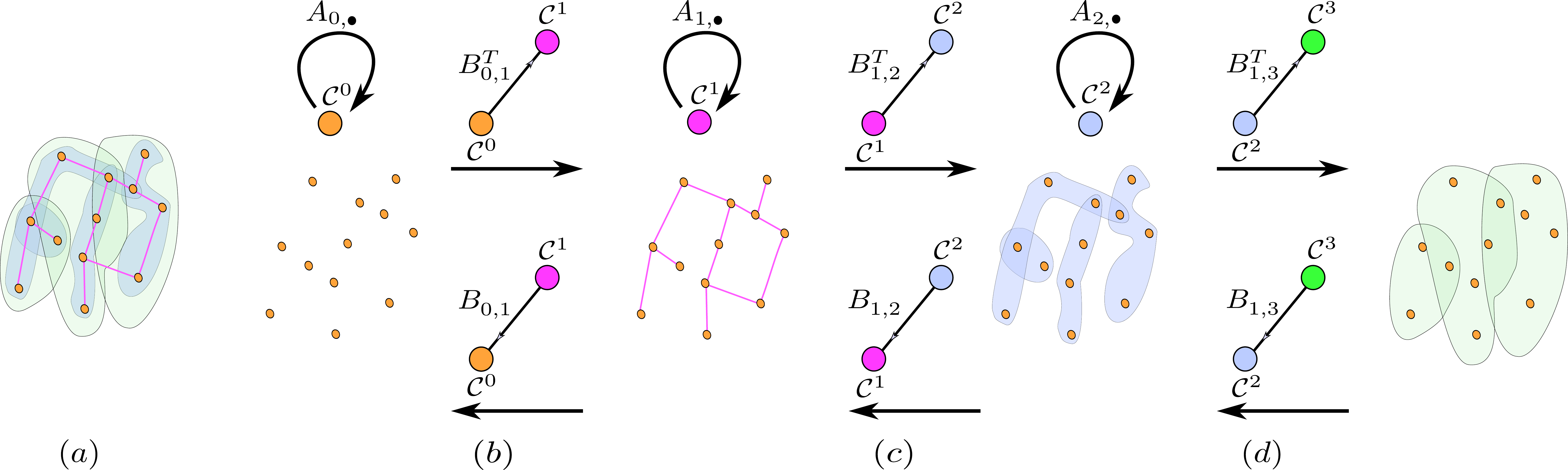}
    \caption{Rank-induced hierarchy.
    Incidence maps transport features from lower to higher ranks in the
    encoder and from higher to lower ranks in the decoder. Within-rank
    refinements act on cochains at a fixed rank, while skip connections merge
    encoder and decoder features at matched ranks.}
\label{fig:topounet_motivation}
\end{figure}

\subsection{TopoU-Net definition}
\label{subsec:topounet-definition}

Let $\mathcal S=(s_0<s_1<\ldots<s_L)$ be an encoder rank path in
$\mathcal R(\mathcal X)$. For readability, write $d_i$ for the feature
dimension used at rank $s_i$. For each $i=0,\ldots,L-1$, let
$T^{\uparrow}_{s_i\to s_{i+1}}:
C^{s_i}(\mathcal X;\mathbb R^{d_i})\to
C^{s_{i+1}}(\mathcal X;\mathbb R^{d_{i+1}})$ be an upward transport and let
$T^{\downarrow}_{s_{i+1}\to s_i}:
C^{s_{i+1}}(\mathcal X;\mathbb R^{d_{i+1}})\to
C^{s_i}(\mathcal X;\mathbb R^{d_i})$ be the corresponding downward transport.
Let $\Phi_{s_{i+1}}$ and $\Psi_{s_i}$ be within-rank refinement maps at ranks
$s_{i+1}$ and $s_i$, respectively. Let $\Omega_{s_L}$ be a bottleneck map at
rank $s_L$, and let
$M_{s_i}:C^{s_i}(\mathcal X;\mathbb R^{d_i})\times
C^{s_i}(\mathcal X;\mathbb R^{d_i})\to
C^{s_i}(\mathcal X;\mathbb R^{d_i})$ be a same-rank merge map.

\begin{definition}[TopoU-Net]
\label{def:TopoU-Net}
Given an input cochain $H^{(s_0)}\in C^{s_0}(\mathcal X;\mathbb R^{d_0})$,
a TopoU-Net first sets $E_{s_0}=H^{(s_0)}$ and computes encoder states
\begin{align}
E_{s_{i+1}}
&=
\Phi_{s_{i+1}}\left(
T^{\uparrow}_{s_i\to s_{i+1}}(E_{s_i})
\right),
~i=0,\ldots,L-1.
\label{eq:topounet-encoder}
\end{align}
The decoder is initialized by $D_{s_L}=\Omega_{s_L}(E_{s_L})$ and then computes
\begin{align}
\widetilde D_{s_i}
&=
\Psi_{s_i}\!\left(
T^{\downarrow}_{s_{i+1}\to s_i}(D_{s_{i+1}})
\right),
~
D_{s_i}
=
M_{s_i}(E_{s_i},\widetilde D_{s_i}),
~i=L-1,\ldots,0.
\label{eq:topounet-decoder}
\end{align}
The output at the input rank is $D_{s_0}$, optionally followed by a
task-specific prediction head.
\end{definition}

The definition separates three operations. Transport changes the rank and is
implemented from incidence structure. Refinement acts within a fixed rank.
Merge maps combine encoder and decoder features only after both are indexed by
the same cells. This is the topological analogue of the alignment condition in
classical U-Nets.

\begin{proposition}[Structural compatibility]
\label{thm:compatibility}
For every $i=0,\ldots,L$, the encoder state $E_{s_i}$ lies in
$C^{s_i}(\mathcal X;\mathbb R^{d_i})$. For every $i=0,\ldots,L-1$, the decoder
state $\widetilde D_{s_i}$ and the merged state $D_{s_i}$ also lie in
$C^{s_i}(\mathcal X;\mathbb R^{d_i})$. Hence every skip merge in
Definition~\ref{def:TopoU-Net} is type-compatible by construction.
\end{proposition}

We next record the equivariance property used throughout the paper. Let
$P_r$ denotes a permutation matrix reordering the cells of rank $r$. Under a
simultaneous reindexing of cells, cochains transform as
$H^{(r)}\mapsto P_rH^{(r)}$, and incidence matrices transform as
$B_{r,r'}\mapsto P_rB_{r,r'}P_{r'}^\top$. A transport family is
reindexing-equivariant if, under this transformation,
$T^{\uparrow}_{r\to r'}(P_rH)=P_{r'}T^{\uparrow}_{r\to r'}(H)$ and
$T^{\downarrow}_{r'\to r}(P_{r'}G)=P_rT^{\downarrow}_{r'\to r}(G)$.
The incidence-convolution transports in Section~\ref{subsec:rank-transport}
satisfy this condition because they use shared feature maps and incidence
matrices only through multiplication by $B_{r,r'}$ or $B_{r,r'}^\top$.

\begin{proposition}[Equivariance]
\label{thm:equivariance}
Assume that all transports in Definition~\ref{def:TopoU-Net} are
reindexing-equivariant, and that the refinement maps $\Phi_{s_i}$, $\Psi_{s_i}$,
the bottleneck map $\Omega_{s_L}$, and the merge maps $M_{s_i}$ are equivariant
to permutations of cells at their respective ranks. Then the TopoU-Net map
$H^{(s_0)}\mapsto D_{s_0}$ is equivariant to joint reindexing of the cells
along $\mathcal S$; if the input is replaced by $P_{s_0}H^{(s_0)}$ and all
incidence matrices are reindexed consistently, then the output is replaced by
$P_{s_0}D_{s_0}$.
\end{proposition}

\subsection{Bottleneck support ratios and skip connections}
\label{subsec:bottleneck-support}

The encoder rank path $\mathcal S=(s_0<s_1<\ldots<s_L)$ determines the support
size at each level of the architecture. Unlike classical U-Nets, where pooling
ratios are design choices, these support sizes are fixed by the number of cells
at the selected ranks. For each step in the encoder, define the support ratio
$\rho_i:=n_{s_{i+1}}/n_{s_i}$, and define the bottleneck support ratio
$\rho_{\mathrm{bot}}:=n_{s_L}/n_{s_0}$. A step with $\rho_i<1$ compresses the
number of cells, while a step with $\rho_i>1$ expands it.

\begin{proposition}[Bottleneck support and capacity]
\label{prop:compression-hierarchy}
Let $\mathcal X$ be a CC and let
$\mathcal S=(s_0<s_1<\ldots<s_L)$ be an encoder rank path. The support ratios
$\rho_i=n_{s_{i+1}}/n_{s_i}$ and $\rho_{\mathrm{bot}}=n_{s_L}/n_{s_0}$ are
determined by the cell counts of $\mathcal X$ and the chosen rank path. Moreover, consider a no-skip, bottleneck-only linear
autoencoder that maps arbitrary inputs in
$C^{s_0}(\mathcal X;\mathbb R^{d_0})$ through a bottleneck
$C^{s_L}(\mathcal X;\mathbb R^{d_L})$ and reconstructs them at rank $s_0$.
A necessary condition for exact reconstruction of arbitrary rank-$s_0$ cochains
is $d_L\geq d_0/\rho_{\mathrm{bot}}$.
\end{proposition}

The role of the capacity condition in Proposition~\ref{prop:compression-hierarchy}
is to identify when a bottleneck is structurally severe. If
$\rho_{\mathrm{bot}}\ll 1$, a no-skip model must either use a much larger
bottleneck feature dimension or discard information that may be needed at the
output rank. Skip connections provide an alternative by bypassing the
compressed ranks and merging encoder features back into the decoder at matched
supports.

\paragraph{Examples.}
For the Texas graph, the triangle bottleneck path $0<1<2$ has
$n_0=183$, $n_1=325$, and $n_2=52$, giving
$\rho_0=325/183\approx1.78$, $\rho_1=52/325\approx0.16$, and
$\rho_{\mathrm{bot}}=52/183\approx0.28$. If the path is extended to a global
rank-$3$ cell, $0<1<2<3$, then $n_3=1$ and
$\rho_{\mathrm{bot}}=1/183\approx0.006$. Thus the global bottleneck is much
more compressive than the triangle bottleneck. By contrast, a $28\times28$
image grid represented by pixels, grid adjacencies, and $2\times2$ patches has
$n_0=784$, $n_1=1512$, and $n_2=729$, so the path $0<1<2$ has
$\rho_{\mathrm{bot}}=729/784\approx0.93$. In this case the bottleneck retains
nearly the same number of cells as the input rank, so reconstruction without
skip connections is structurally less constrained.

\subsection{Canonical instantiation}
\label{subsec:canonical}

The general definition allows different choices of transport, refinement, and
merge maps. The canonical TopoU-Net used in our experiments takes cross-rank
transport to be incidence convolution along the selected encoder rank path
$\mathcal S=(s_0<\ldots<s_L)$. For each consecutive pair of ranks in the path,
$(s_i,s_{i+1})$, let $\bar B_{s_i,s_{i+1}}$ denote either the raw incidence
matrix $B_{s_i,s_{i+1}}$ or a fixed degree-normalized incidence matrix with the
same support. The upward and downward transports are
\[
T^{\uparrow}_{s_i\to s_{i+1}}(H)
=
\phi \left(\bar B_{s_i,s_{i+1}}^\top H W_i^{\uparrow}\right),
~
T^{\downarrow}_{s_{i+1}\to s_i}(G)
=
\phi \left(\bar B_{s_i,s_{i+1}}G W_i^{\downarrow}\right),
\]
where $W_i^{\uparrow}\in\mathbb R^{d_i\times d_{i+1}}$ and
$W_i^{\downarrow}\in\mathbb R^{d_{i+1}\times d_i}$. If $s_{i+1}-s_i>1$, the
same formula uses the direct incidence matrix $B_{s_i,s_{i+1}}$. Within-rank refinements $\Phi_{s_i}$ and $\Psi_{s_i}$ are applied after
transport and do not change the rank or support. They may be implemented as
pointwise MLPs, or as same-rank message-passing blocks using incidence-induced
adjacencies. For active ranks $r<q$, the operator
$A_{r\mid q}=B_{r,q}B_{r,q}^{\top}$ connects $r$-cells that are incident to a
common $q$-cell. For $q<r$, the operator
$A_{r\mid q}=B_{q,r}^{\top}B_{q,r}$ gives the analogous adjacency on
$r$-cells through shared lower-rank cells. These refinements update features
within a fixed cochain space. The hierarchy of the architecture is determined
by the cross-rank incidence transports.

For example, if $\mathcal S=(0<2<3)$, the encoder transports features as
$E_0 \xrightarrow{\bar B_{0,2}^{\top}} E_2
\xrightarrow{\bar B_{2,3}^{\top}} E_3$. The decoder starts from
$D_3=\Omega_3(E_3)$ and reverses the path as
\[
D_3
\xrightarrow{\bar B_{2,3}}
\widetilde D_2
\xrightarrow{M_2(E_2,\cdot)}
D_2
\xrightarrow{\bar B_{0,2}}
\widetilde D_0
\xrightarrow{M_0(E_0,\cdot)}
D_0 .
\]
The skip merges occur only at matched ranks, here ranks $2$ and $0$, where the
encoder and decoder states are indexed by the same cells.

The merge operator must act within a common cochain space. In the default
additive merge, the decoder state at rank $s_i$ is
$D_{s_i}=\phi((E_{s_i}+\widetilde D_{s_i})W_i^m)$, where
$W_i^m\in\mathbb R^{d_i\times d_i}$ is shared across cells of rank $s_i$.
Other equivariant within-rank merges, such as concatenation or gated fusion,
can be used without changing the definition. In no-skip ablations, the merge is
removed and the decoder uses $D_{s_i}=\widetilde D_{s_i}$. This isolates the
effect of bypassing compressed bottleneck ranks.

\section{Experimental study}
\label{sec:experiments}

We evaluate TopoU-Net on node classification, graph classification, hypergraph
node classification, 3D point-cloud classification, and image segmentation.
Image reconstruction results are reported in Appendix~\ref{app:image-reconstruction}.

\paragraph{Rank structure construction.}
Each domain is represented as a combinatorial complex with domain-specific rank
assignment. Graphs use TopoNetX~\cite{hajij2024topox} to construct nodes
(rank~0), edges (rank~1), and triangles as maximal 3-cliques (rank~2). For
heterophilic graph datasets, we additionally use a single global rank-3 cell
incident to all rank-2 triangle cells. This is the construction used in the
compression tables, where $n_3=1$. Hypergraphs use native node--hyperedge
incidence as the $0\to1$ transport operator
$H\in\{0,1\}^{|V|\times|\mathcal E|}$, with rank~2 cells formed from triples
of hyperedges satisfying a fixed pairwise-overlap criterion. For point-cloud
classification on ModelNet10/40~\cite{wu20153d}, each shape is sampled into
512 points. Sampled points are treated as rank~0 cells, symmetric kNN
adjacencies as rank~1 cells, and local kNN clique triangles as inferred rank~2
cells.

\paragraph{Rank path selection.}
Rank paths are selected based on dataset structure rather than using a single
architecture across all datasets. For homophilic graphs (Computers, Photo), the
path $0\to1\to0$ is used because label smoothness can be captured through edge
neighborhoods. For heterophilic graphs (Cornell, Wisconsin, Texas, Actor,
Chameleon, Squirrel), the path
$0\to1\to2\to3\to2\to1\to0$ introduces triangle and global aggregation channels
in addition to edge neighborhoods. Graph classification, hypergraph node
classification, and point-cloud classification use the path
$0\to1\to2\to1\to0$. The relevant structural quantity is the bottleneck support
ratio $\rho_{\mathrm{bot}}$ from Proposition~\ref{prop:compression-hierarchy},
analyzed in Section~\ref{sec:compression-analysis}.

\subsection{Node classification}

Table~\ref{tab:node_cls} reports accuracy across eight benchmarks compared
against GCN~\cite{kipf2017semi}, GraphSAGE~\cite{hamilton2017inductive},
GAT~\cite{velickovic2018graph}, GIN~\cite{xu2019powerful},
H2GCN~\cite{zhu2020beyond}, MixHop~\cite{abu2019mixhop},
DiffPool~\cite{ying2018diffpool}, and Graph U-Net~\cite{gao2019graph}.
TopoU-Net achieves the highest mean accuracy on six of eight datasets. The
largest gains occur on heterophilic benchmarks: $+4.9$ percentage points on
Cornell, $+5.7$ on Wisconsin, and $+9.8$ on Texas over the next strongest
baseline in each case.

These results are consistent with the mechanism described in
Section~\ref{subsec:bottleneck-support}. On heterophilic graphs, immediate edge
neighborhoods can be less reliable for label prediction, so aggregation through
triangles and the global rank-3 cell can provide useful additional context. On
the homophilic co-purchase graphs, Computers and Photo, GraphSAGE remains the
strongest evaluated baseline. This is consistent with strong local homophily,
where edge-neighborhood aggregation is already effective and higher-order rank
paths are less important.

\begin{table*}[h]
\centering
\caption{Test accuracy (\%) across node classification benchmarks.}
\label{tab:node_cls}
\resizebox{1\textwidth}{!}{%
\begin{tabular}{lcccccccc}
\hline
\textbf{Model} 
& \textbf{Computers} & \textbf{Photo} & \textbf{Actor} 
& \textbf{Chameleon} & \textbf{Squirrel} & \textbf{Cornell}
& \textbf{Wisconsin} & \textbf{Texas} \\
\hline
GCN & $66.4 \pm 4.0$ & $77.2 \pm 4.9$ & $33.1 \pm 0.9$
& $41.1 \pm 2.6$ & $28.1 \pm 0.9$ & $47.8 \pm 9.3$
& $62.1 \pm 7.3$ & $62.9 \pm 4.3$ \\
GraphSAGE & $\mathbf{80.4 \pm 0.8}$ & $\mathbf{89.1 \pm 1.5}$ 
& $33.7 \pm 0.8$ & $47.3 \pm 2.8$ & $31.7 \pm 1.7$ & $70.5 \pm 5.1$
& $82.7 \pm 3.1$ & $81.0 \pm 5.8$ \\
GAT & $74.4 \pm 2.3$ & $85.7 \pm 2.5$ & $29.3 \pm 1.0$
& $44.8 \pm 2.8$ & $26.1 \pm 0.5$ & $46.4 \pm 8.5$
& $61.1 \pm 6.3$ & $59.4 \pm 4.5$ \\
GIN & $73.6 \pm 2.9$ & $81.8 \pm 3.7$ & $23.7 \pm 1.6$
& $48.2 \pm 2.4$ & $31.5 \pm 1.9$ & $52.9 \pm 7.1$
& $62.5 \pm 3.3$ & $64.3 \pm 7.9$ \\
H2GCN & $46.0 \pm 2.2$ & $75.2 \pm 2.0$ & $30.0 \pm 1.9$
& $37.7 \pm 3.2$ & $24.4 \pm 1.6$ & $72.4 \pm 6.9$
& $81.7 \pm 3.2$ & $79.1 \pm 6.2$\\
MixHop & $67.6 \pm 2.4$ & $78.1 \pm 4.0$ & $29.8 \pm 2.0$
& $36.4 \pm 1.7$ & $23.2 \pm 1.9$ & $71.3 \pm 2.4$
& $83.1 \pm 3.0$ & $75.1 \pm 5.9$\\
DiffPool & $57.7 \pm 12$ & $82.4 \pm 4.8$ & $25.0 \pm 2.2$
& $49.3 \pm 2.6$ & $33.7 \pm 1.4$ & $48.6 \pm 6.9$
& $58.6 \pm 5.0$ & $59.1 \pm 5.3$\\
Graph U-Net & $71.2 \pm 1.0$ & $54.4 \pm 6.6$ & $35.9 \pm 1.0$
& $41.3 \pm 1.4$ & $33.5 \pm1.6$ & $46.7 \pm 6.0$
& $60.3 \pm 6.1$ & $60.0 \pm 6.9$  \\
\hline 
TopoU-Net & $73.6 \pm 1.9$ & $83.2 \pm 1.1$ & $\mathbf{37.5 \pm 0.9}$
& $\mathbf{49.6 \pm 1.7}$ & $\mathbf{35.1 \pm 1.7}$ 
& $\mathbf{77.3 \pm 4.7}$ & $\mathbf{88.8 \pm 5.6}$ 
& $\mathbf{90.8 \pm 4.0}$ \\
\hline
\end{tabular}}
\end{table*}

\subsection{Graph classification}

Table~\ref{tab:graph_cls} reports graph classification accuracy on MUTAG,
PROTEINS, and IMDB-BINARY using the rank path
$0\!\to\!1\!\to\!2\!\to\!1\!\to\!0$ with global mean pooling. TopoU-Net
achieves the highest mean accuracy on MUTAG and IMDB-BINARY and remains competitive
on PROTEINS.

\begin{wraptable}[8]{r}{0.5\textwidth}
\centering
\caption{Test@BestVal accuracy (\%) on graph classification.}
\label{tab:graph_cls}
\footnotesize
\setlength{\tabcolsep}{4pt}
\renewcommand{\arraystretch}{0.78}
\begin{tabular}{lccc}
\toprule
\textbf{Model} & \textbf{MUTAG} & \textbf{PROTEINS} & \textbf{IMDB} \\
\midrule
GCN         & $84.5{\pm}7.6$ & $74.5{\pm}5.5$ & $48.8{\pm}4.1$ \\
GraphSAGE   & $82.9{\pm}8.1$ & $74.2{\pm}5.2$ & $49.8{\pm}4.8$ \\
GAT         & $71.8{\pm}3.7$ & $\mathbf{74.7{\pm}4.0}$ & $48.3{\pm}3.1$ \\
GIN         & $91.4{\pm}5.1$ & $73.4{\pm}3.9$ & $69.2{\pm}3.7$ \\
Graph U-Net & $82.9{\pm}9.5$ & $\mathbf{74.7{\pm}3.9}$ & $50.5{\pm}3.6$ \\
\midrule
TopoU-Net   & $\mathbf{91.5{\pm}7.9}$ & $73.4{\pm}4.2$ & $\mathbf{71.1{\pm}3.3}$ \\
\bottomrule
\end{tabular}
\end{wraptable}

These results suggest that rank-2 structure is useful for
graph-level prediction in domains with higher-order substructure, such as
molecular graphs and interaction networks.

\subsection{Hypergraph node classification}

Given a hypergraph $\mathcal H=(V,\mathcal E)$ with incidence matrix
$H\in\{0,1\}^{|V|\times|\mathcal E|}$, we use $H$ as the transport operator
between ranks 0 and 1. Rank-2 cells are constructed from triples of hyperedges
satisfying the fixed overlap criterion, giving incidence matrix $B_{1,2}$ and
the path $0\to1\to2\to1\to0$ with skip connections at ranks 0 and 1.
Table~\ref{tab:hypergraph_results} reports accuracy on co-citation and
co-authorship hypergraph benchmarks~\cite{yadati2019hypergcn}, implemented
using the DHG library~\cite{feng2019hypergraph}. TopoU-Net attains the highest
mean accuracy on four of five datasets compared with MLP,
HGNN~\cite{feng2019hypergraph,gao2022hgnnplus},
HyperSAGE~\cite{arya2020hypersage}, and HNHN~\cite{dong2020hnhn}.

The MLP is competitive on several co-citation datasets, suggesting that node
features are already informative. TopoU-Net improves over the MLP on all
reported hypergraph datasets, indicating that the rank path can add useful
higher-order signal. The largest gains occur on co-authorship datasets, where
hyperedges directly encode group interactions.

\begin{table}[htbp]
\centering
\caption{Hypergraph node classification accuracy (\%).}
\label{tab:hypergraph_results}
\small
\resizebox{\linewidth}{!}{%
\begin{tabular}{lccccc}
\toprule
\textbf{Model} & \textbf{CocitationCora} & \textbf{CocitationCiteseer} 
& \textbf{CocitationPubmed} & \textbf{CoauthorshipCora} 
& \textbf{CoauthorshipDBLP} \\
\midrule
MLP & $46.15 \pm 1.93$ & $47.65 \pm 0.94$ & $36.13 \pm 8.59$ 
& $46.15 \pm 1.93$ & $78.55 \pm 0.23$ \\
HGNN & $42.79 \pm 1.01$ & $35.83 \pm 0.70$ & $22.32 \pm 2.14$ 
& $59.36 \pm 1.20$ & $87.65 \pm 0.11$ \\
HyperSAGE & $44.58 \pm 0.72$ & $44.44 \pm 0.91$ 
& $\mathbf{39.83 \pm 0.31}$ & $47.10 \pm 0.47$ & $82.64 \pm 0.34$ \\
HNHN & $47.94 \pm 1.79$ & $45.72 \pm 0.64$ & $36.81 \pm 1.38$ 
& $55.97 \pm 1.82$ & $86.12 \pm 0.46$ \\
\midrule
TopoU-Net & $\mathbf{51.63 \pm 1.80}$ & $\mathbf{48.73 \pm 1.28}$ 
& $39.67 \pm 0.37$ & $\mathbf{63.57 \pm 1.62}$ 
& $\mathbf{87.81 \pm 0.16}$ \\
\bottomrule
\end{tabular}}
\end{table}

\subsection{Point-cloud classification}

\begin{wraptable}{r}{0.49\textwidth}
\vspace{-1.5em}
\centering
\caption{Point-cloud classification accuracy.}
\label{tab:pointcloud_cls}
\small
\resizebox{0.40\textwidth}{!}{%
\begin{tabular}{lcc}
\toprule
\textbf{Model} & \textbf{ModelNet10} & \textbf{ModelNet40} \\
\midrule
PointNet  & $0.8893{\pm}0.0008$ & $0.8258{\pm}0.0006$ \\
GraphSAGE & $\mathbf{0.8959{\pm}0.0008}$ & $0.8179{\pm}0.0066$ \\
TopoU-Net & $0.8855{\pm}0.0062$ & $\mathbf{0.8318{\pm}0.0074}$ \\
\bottomrule
\end{tabular}}
\vspace{-1.0em}
\end{wraptable}

We evaluate TopoU-Net on ModelNet10/40~\cite{wu20153d} in a point-cloud
setting. PointNet operates directly on the point set, GraphSAGE uses a kNN
graph, and TopoU-Net lifts the point cloud to a rank-$0/1/2$ complex with
sampled points as $0$-cells, symmetric kNN edges as $1$-cells, and local kNN
clique triangles as inferred $2$-cells. Using the path
$0\to1\to2\to1\to0$, TopoU-Net achieves the highest mean accuracy on
ModelNet40 and remains competitive on ModelNet10.

\subsection{Image segmentation}
\label{subsec:image-segmentation}

For image segmentation, we evaluate a learned-transport variant inspired by
TopoU-Net. In this variant, the rank-transport matrices are learned soft
incidence maps rather than fixed incidence matrices from an explicit
combinatorial complex. Given pixel features, the model transports information
from pixels to latent topological slots and back. We report two variants:
TopoU-Net Soft, a small learned-transport model without U-Net skips, and
TopoU-Net-S2, a compact skip-connected multi-kernel version. Results are shown
on Oxford-IIIT Pet~\cite{parkhi2012cats} and Pascal VOC
2012~\cite{everingham2010pascal}.

\begin{table}[h]
\centering
\caption{Segmentation accuracy averaged across runs.}
\label{tab:image-segmentation-accuracy}
\small
\begin{tabular}{lccc}
\toprule
\textbf{Model} & \textbf{Params} & \textbf{Oxford-IIIT Pet Acc.} & \textbf{Pascal VOC 2012 Acc.} \\
\midrule
U-Net & 1.93M & $0.8508 \pm 0.0030$ & $0.7440 \pm 0.0002$ \\
ResUNet & 2.88M & $0.8505 \pm 0.0005$ & $0.7421 \pm 0.0002$ \\
Attention U-Net & 1.95M & $\mathbf{0.8545 \pm 0.0025}$ & $0.7449 \pm 0.0021$ \\
DeepLab-lite & 2.06M & $\mathbf{0.8545 \pm 0.0018}$ & $0.7421 \pm 0.0041$ \\
\midrule
TopoU-Net Soft & 0.42M & $0.8440 \pm 0.0023$ & $0.7360 \pm 0.0018$ \\
TopoU-Net-S2 & 0.97M & $0.8449 \pm 0.0028$ & $\mathbf{0.7453 \pm 0.0019}$ \\
\bottomrule
\end{tabular}
\end{table}

The learned-transport variants give competitive segmentation accuracy in a
smaller parameter regime. TopoU-Net Soft uses $0.42$M parameters and remains
close to the classical baselines on Oxford-IIIT Pet. TopoU-Net-S2 improves over
the soft version while using roughly half the parameters of U-Net, and obtains
the highest Pascal VOC accuracy among the reported models.

\subsection{Ablation studies}

We study how bottleneck support, rank-path complexity, and skip connections
affect performance.

\subsubsection{Compression profiles}
\label{sec:compression-analysis}

Table~\ref{tab:compression-profile} reports support profiles along the encoder
rank paths. We use $\rho_i=n_{s_{i+1}}/n_{s_i}$ for consecutive encoder steps
and $\rho_{\mathrm{bot}}=n_{s_L}/n_{s_0}$ for the bottleneck support ratio.
Ratios below one indicate support compression; ratios above one indicate
support expansion.

\begin{table*}[h]
\centering
\caption{Support profiles along encoder rank paths. The full U-shaped decoder
path reverses the encoder path shown here.}
\label{tab:compression-profile}
\resizebox{\textwidth}{!}{%
\begin{tabular}{lccccccccl}
\hline
\textbf{Dataset} & $n_0$ & $n_1$ & $n_2$ & $n_3$ & $\rho_0$ & $\rho_1$ 
& $\rho_2$ & $\rho_{\text{total}}$ & \textbf{Path} \\
\hline
\multicolumn{10}{l}{\emph{Heterophilic}} \\
Texas & 183 & 325 & 52 & 1 & 1.78 & 0.16 & 0.02 & 0.006 
& $0\!\to\!1\!\to\!2\!\to\!3\!\to\!2\!\to\!1\!\to\!0$ \\
Wisconsin & 251 & 515 & 89 & 1 & 2.05 & 0.17 & 0.01 & 0.004 
& $0\!\to\!1\!\to\!2\!\to\!3\!\to\!2\!\to\!1\!\to\!0$ \\
Cornell & 183 & 298 & 47 & 1 & 1.63 & 0.16 & 0.02 & 0.005 
& $0\!\to\!1\!\to\!2\!\to\!3\!\to\!2\!\to\!1\!\to\!0$ \\
Chameleon & 2277 & 36101 & 1507 & 1 & 15.9 & 0.04 & 0.001 & 0.0004 
& $0\!\to\!1\!\to\!2\!\to\!3\!\to\!2\!\to\!1\!\to\!0$ \\
\hline
\multicolumn{10}{l}{\emph{Homophilic}} \\
Computers & 13752 & 245861 & 8943 & --- & 17.9 & 0.04 & --- & 0.65 
& $0\!\to\!1\!\to\!0$ \\
Photo & 7650 & 119081 & 4521 & --- & 15.6 & 0.04 & --- & 0.59 
& $0\!\to\!1\!\to\!0$ \\
\hline
\multicolumn{10}{l}{\emph{Grid}} \\
$28\times 28$ & 784 & 1512 & 729 & --- & 1.93 & 0.48 & --- & 0.93 
& $0\!\to\!1\!\to\!2\!\to\!1\!\to\!0$ \\
\hline
\end{tabular}}
\end{table*}

The heterophilic paths ending in a global rank-3 cell have
$\rho_{\mathrm{bot}}<0.01$, so the bottleneck support is extremely compressed.
The grid path has $\rho_{\mathrm{bot}}\approx0.93$, so the bottleneck retains
nearly the same number of cells as the input rank. The homophilic node-edge
paths are support-expanding at the bottleneck; they should not be described as
moderate-compression paths.

\subsubsection{Rank path complexity}

Table~\ref{tab:rank-path-complexity} varies rank-path complexity on
heterophilic datasets. Increasing the path from $0\to1\to0$ to
$0\to1\to2\to3\to2\to1\to0$ improves performance, with gains up to $26.3$
percentage points on Texas. This comparison is not parameter-matched: adding
higher ranks also increases depth and parameter count. The result should
therefore be interpreted as evidence for the practical value of richer rank
paths, not as an isolated proof that topology alone explains the gains.

\begin{table*}[h]
\centering
\caption{Node accuracy (\%) versus rank-path complexity on heterophilic datasets.}
\label{tab:rank-path-complexity}
\resizebox{\textwidth}{!}{%
\begin{tabular}{llcccccccc}
\hline
\textbf{Path} & \textbf{Cells} & $\rho_{\text{total}}$ 
& \textbf{Params} & \textbf{Actor} & \textbf{Chameleon} 
& \textbf{Squirrel} & \textbf{Cornell} & \textbf{Wisconsin} 
& \textbf{Texas} \\
\hline
$0\!\to\!1\!\to\!0$ & Nodes, edges & 1.0 & 15K 
& $29.7{\pm}2.6$ & $32.9{\pm}4.0$ & $25.6{\pm}1.3$ 
& $53.2{\pm}8.6$ & $71.3{\pm}4.4$ & $64.5{\pm}6.9$ \\
$0\!\to\!1\!\to\!2\!\to\!1\!\to\!0$ & +Triangles & 0.28 & 28K 
& $34.9{\pm}5.7$ & $45.5{\pm}2.1$ & $33.1{\pm}2.2$ 
& $56.2{\pm}7.8$ & $73.3{\pm}4.3$ & $69.1{\pm}7.0$ \\
$0\!\to\!1\!\to\!2\!\to\!3\!\to\!2\!\to\!1\!\to\!0$ & +Global & 0.006 & 45K 
& $\mathbf{37.5{\pm}0.9}$ & $\mathbf{49.6{\pm}1.7}$ 
& $\mathbf{35.1{\pm}1.7}$ & $\mathbf{77.3{\pm}4.7}$ 
& $\mathbf{88.8{\pm}5.6}$ & $\mathbf{90.8{\pm}4.0}$ \\
\hline
\end{tabular}}
\end{table*}

\subsubsection{Skip connections}
\label{sec:skip-ablation}

Table~\ref{tab:skip-ablation} tests the prediction of
Proposition~\ref{prop:compression-hierarchy}: when the bottleneck support is
severely compressed, removing skip connections should be especially damaging.
This trend is visible on Texas, Wisconsin, and Cornell, where
$\rho_{\mathrm{bot}}<0.01$ and removing skips reduces accuracy by $15$--$24$
percentage points. The node-edge homophilic paths are support-expanding rather
than compressive, and show smaller degradation. For the grid reconstruction
row, the metric is MSE rather than accuracy, so the reported effect is a
relative error increase.

\begin{table*}[h]
\centering
\caption{Effect of removing skip connections. For accuracy rows, $\Delta$ is
the no-skip result minus the with-skip result in percentage points. For the MSE
row, $\Delta$ is the relative MSE increase.}
\label{tab:skip-ablation}
\resizebox{0.85\textwidth}{!}{%
\begin{tabular}{llcccc}
\hline
\textbf{Dataset} & \textbf{Path} & $\rho_{\text{total}}$ 
& \textbf{With Skip} & \textbf{No Skip} & $\Delta$ \\
\hline
\multicolumn{6}{l}{\emph{Severe compression}} \\
Texas & $0\!\to\!1\!\to\!2\!\to\!3\!\to\!2\!\to\!1\!\to\!0$ & 0.006 
& $90.8{\pm}4.0$ & $67.3{\pm}6.2$ & $-23.5$ \\
Wisconsin & $0\!\to\!1\!\to\!2\!\to\!3\!\to\!2\!\to\!1\!\to\!0$ & 0.004 
& $88.8{\pm}5.6$ & $71.1{\pm}5.8$ & $-17.7$ \\
Cornell & $0\!\to\!1\!\to\!2\!\to\!3\!\to\!2\!\to\!1\!\to\!0$ & 0.005 
& $77.3{\pm}4.7$ & $61.4{\pm}7.1$ & $-15.9$ \\
\hline
\multicolumn{6}{l}{\emph{Moderate compression}} \\
Computers & $0\!\to\!1\!\to\!0$ & 0.65 
& $73.6{\pm}1.9$ & $68.2{\pm}2.4$ & $-5.4$ \\
Photo & $0\!\to\!1\!\to\!0$ & 0.59 
& $83.2{\pm}1.1$ & $79.1{\pm}1.8$ & $-4.1$ \\
\hline
\multicolumn{6}{l}{\emph{Minimal compression}} \\
Grid $28\times 28$ & $0\!\to\!1\!\to\!2\!\to\!1\!\to\!0$ & 0.93 
& $9.5\times 10^{-4}$ & $1.2\times 10^{-3}$ & $-26\%$ \\
\hline
\end{tabular}}
\end{table*}
\subsubsection{Minimal rank-path baseline}

We compare the dataset-specific TopoU-Net to a minimal node-edge path
$0\!\to\!1\!\to\!0$. This tests whether performance comes simply from using
incidence-based encoder--decoder structure, or from selecting appropriate
higher-order rank paths. Table~\ref{tab:minimal_vs_topounet} shows that dataset-specific paths improve
all datasets, with the largest gains on heterophilic graphs. This supports the
view that rank-path selection is a central part of the TopoU-Net design, not
just an implementation detail.
\begin{table*}[h]
\centering
\caption{Minimal node-edge baseline versus dataset-specific TopoU-Net.}
\label{tab:minimal_vs_topounet}
\resizebox{\textwidth}{!}{%
\begin{tabular}{lcccccccc}
\hline
\textbf{Model} & \textbf{Computers} & \textbf{Photo} & \textbf{Actor} 
& \textbf{Chameleon} & \textbf{Squirrel} & \textbf{Cornell} 
& \textbf{Wisconsin} & \textbf{Texas} \\
\hline
Minimal $0\!\to\!1\!\to\!0$
& $73.6{\pm}1.9$ & $83.2{\pm}1.1$ & $29.7{\pm}2.6$ 
& $32.9{\pm}4.0$ & $25.6{\pm}1.3$ & $53.2{\pm}8.6$ 
& $71.3{\pm}4.7$ & $64.5{\pm}6.9$ \\
Dataset-specific
& $\mathbf{73.6{\pm}1.9}$ & $\mathbf{83.2{\pm}1.1}$ 
& $\mathbf{37.5{\pm}0.9}$ & $\mathbf{45.5{\pm}2.1}$ 
& $\mathbf{35.1{\pm}1.7}$ & $\mathbf{77.3{\pm}4.7}$ 
& $\mathbf{88.8{\pm}5.6}$ & $\mathbf{90.8{\pm}4.0}$ \\
\hline
\end{tabular}}
\end{table*}

\section{Conclusion}
\label{sec:conclusion}
We introduced TopoU-Net, a U-Net-style encoder-decoder for combinatorial
complexes. Ranked cells define hierarchy levels, incidence maps provide
cross-rank transport, and skip connections merge encoder and decoder features
at matched ranks. This gives a common formulation for graphs, hypergraphs,
meshes, and grids, while making rank-path compression explicit through the
bottleneck support ratio. Experiments and ablations support the proposed view:
higher-order rank paths help when pairwise neighborhoods are insufficient, and
skip connections become increasingly important under severe bottleneck
compression. The main limitation is that rank paths and higher-order cells are currently
chosen by hand. This improves interpretability but makes performance dependent
on preprocessing and task-specific choices. Scalability is another limitation:
constructing and storing triangles, hyperedge overlaps, or global cells can be
expensive on large domains. Future work should learn or select rank paths
automatically and develop sparse or sampled higher-order constructions.




\bibliographystyle{plain}   


\appendix
\section{Proofs and additional theoretical elaborations}

\subsection{Additional rank-transport parameterizations}
\label{app:transport-parameterizations}

This appendix lists several transport parameterizations compatible with the
rank-transport definition in Section~\ref{subsec:rank-transport}. The important
constraint is that any attention or gating weights must be computed by shared
functions from equivariant cell features, or be tied across incidence types.
Arbitrary cell-indexed or incidence-indexed parameters are generally
not permutation-equivariant.

For $y\in\mathcal X^{r'}$, define the incident lower-rank neighborhood
$\mathcal I_r(y) := \{x\in\mathcal X^r:x\subset y\}$. The downward variants
are obtained by reversing the incidence direction.

\begin{table}[h]
\centering
\small
\begin{tabular}{lll}
\toprule
\textbf{Parameterization}
& \textbf{Upward update for } $y\in\mathcal X^{r'}$
& \textbf{Condition} \\
\midrule
Incidence convolution
&
$\displaystyle
\phi\!\left(\sum_{x\in\mathcal I_r(y)} H_x W\right)
$
&
Shared $W$
\\[1.2em]

Degree-normalized incidence
&
$\displaystyle
\phi\!\left(\sum_{x\in\mathcal I_r(y)} \bar B_{xy} H_x W\right)
$
&
Fixed incidence weights $\bar B_{xy}$
\\[1.2em]

Attention-weighted incidence
&
$\displaystyle
\phi\!\left(\sum_{x\in\mathcal I_r(y)} \alpha_{xy} H_x W\right)
$
&
$\alpha_{xy}$ from shared equivariant scoring
\\[1.2em]

Gated incidence
&
$\displaystyle
\phi\!\left(\sum_{x\in\mathcal I_r(y)} g_{xy} H_x W\right)
$
&
$g_{xy}$ from shared equivariant gating
\\
\bottomrule
\end{tabular}
\caption{Representative rank-transport parameterizations. Here $H_x$ is the
feature vector on the source cell $x$, $W$ is shared across cells, and
$\bar B$ denotes a fixed raw or normalized incidence matrix. Attention and
gating preserve equivariance only when their weights are produced by shared
functions of equivariant features, not by arbitrary cell-indexed parameters.}
\label{tab:transport-appendix}
\end{table}

\subsection{Proofs for Section~\ref{sec:topounet}}
\label{app:topounet-proofs}

\begin{proof}[Proof of Proposition~\ref{thm:compatibility}]
The claim follows by induction over the encoder and decoder recursions. The
input satisfies $E_{s_0}\in C^{s_0}(\mathcal X;\mathbb R^{d_0})$ by definition.
If $E_{s_i}\in C^{s_i}(\mathcal X;\mathbb R^{d_i})$, then
$T^{\uparrow}_{s_i\to s_{i+1}}(E_{s_i})$ lies in
$C^{s_{i+1}}(\mathcal X;\mathbb R^{d_{i+1}})$, and the within-rank map
$\Phi_{s_{i+1}}$ preserves this cochain space. Hence
$E_{s_{i+1}}\in C^{s_{i+1}}(\mathcal X;\mathbb R^{d_{i+1}})$.

For the decoder, $D_{s_L}=\Omega_{s_L}(E_{s_L})$ lies in
$C^{s_L}(\mathcal X;\mathbb R^{d_L})$. If
$D_{s_{i+1}}\in C^{s_{i+1}}(\mathcal X;\mathbb R^{d_{i+1}})$, then
$T^{\downarrow}_{s_{i+1}\to s_i}(D_{s_{i+1}})$ lies in
$C^{s_i}(\mathcal X;\mathbb R^{d_i})$, and $\Psi_{s_i}$ preserves that space.
Thus $\widetilde D_{s_i}\in C^{s_i}(\mathcal X;\mathbb R^{d_i})$. Since
$E_{s_i}$ lies in the same cochain space, the merge map $M_{s_i}$ is
well-defined and returns $D_{s_i}\in C^{s_i}(\mathcal X;\mathbb R^{d_i})$.
\end{proof}

\begin{proof}[Proof of Proposition~\ref{thm:equivariance}]
Let primed variables denote the states obtained after reindexing cells and
incidence matrices. We prove by induction that $E'_{s_i}=P_{s_i}E_{s_i}$ and
$D'_{s_i}=P_{s_i}D_{s_i}$ for all ranks in the path. The encoder claim holds
at $s_0$ by the definition of the reindexed input. If
$E'_{s_i}=P_{s_i}E_{s_i}$, reindexing-equivariance of the transport gives
$T'^{\uparrow}_{s_i\to s_{i+1}}(E'_{s_i})
=P_{s_{i+1}}T^{\uparrow}_{s_i\to s_{i+1}}(E_{s_i})$. Equivariance of
$\Phi_{s_{i+1}}$ then gives $E'_{s_{i+1}}=P_{s_{i+1}}E_{s_{i+1}}$.

At the bottleneck, equivariance of $\Omega_{s_L}$ gives
$D'_{s_L}=P_{s_L}D_{s_L}$. Assume $D'_{s_{i+1}}=P_{s_{i+1}}D_{s_{i+1}}$.
Reindexing-equivariance of the downward transport and equivariance of
$\Psi_{s_i}$ imply $\widetilde D'_{s_i}=P_{s_i}\widetilde D_{s_i}$. Since the
merge map is equivariant and $E'_{s_i}=P_{s_i}E_{s_i}$, we obtain
$D'_{s_i}=P_{s_i}D_{s_i}$. Taking $i=0$ proves the result.
\end{proof}

\begin{proof}[Proof of Proposition~\ref{prop:compression-hierarchy}]
The first claim follows directly from the definitions. Once the complex
$\mathcal X$ and rank path $\mathcal S$ are fixed, each support size
$n_{s_i}=|\mathcal X^{s_i}|$ is fixed, and hence so are
$\rho_i=n_{s_{i+1}}/n_{s_i}$ and
$\rho_{\mathrm{bot}}=n_{s_L}/n_{s_0}$.

For the capacity claim, a no-skip linear autoencoder whose output depends on
the input only through the bottleneck factors through a vector space of
dimension $n_{s_L}d_L$. Exact reconstruction of arbitrary inputs in
$C^{s_0}(\mathcal X;\mathbb R^{d_0})\cong\mathbb R^{n_{s_0}d_0}$ requires the
linear encoder into the bottleneck to be injective. A linear injective map from
$\mathbb R^{n_{s_0}d_0}$ to $\mathbb R^{n_{s_L}d_L}$ can exist only if
$n_{s_L}d_L\geq n_{s_0}d_0$. Dividing by $n_{s_L}$ gives
$d_L\geq d_0/\rho_{\mathrm{bot}}$.
\end{proof}

\section{Additional experimental results}

\subsection{Parameter efficiency}

Table~\ref{tab:param_efficiency} reports parameter counts for MNIST
reconstruction. TopoU-Net uses roughly $8.7$--$9.4\times$ fewer parameters than
the convolutional encoder-decoder baselines in this comparison, while also
achieving the lowest reported MSE.

\begin{table}[h]
\centering
\caption{Parameters and MNIST reconstruction MSE.}
\label{tab:param_efficiency}
\begin{tabular}{lcc}
\hline
\textbf{Model} & \textbf{Parameters} & \textbf{MSE} \\
\hline
UNet & $1.05 \times 10^5$ & $1.63 \times 10^{-3}$ \\
ResUNet & $1.08 \times 10^5$ & $1.58 \times 10^{-3}$ \\
AttentionUNet & $1.07 \times 10^5$ & $1.90 \times 10^{-3}$ \\
UNet++ & $1.14 \times 10^5$ & $1.53 \times 10^{-3}$ \\
\hline
TopoU-Net & $\mathbf{1.21 \times 10^4}$ & $\mathbf{9.5 \times 10^{-4}}$ \\
\hline
\end{tabular}
\end{table}

\subsection{Image reconstruction}
\label{app:image-reconstruction}

Image reconstruction is a useful stress test because pixel grids are the domain
where convolutional U-Nets are most natural. All inputs are resized to
$28\times 28$ grayscale. Table~\ref{tab:reconstruction_mse} reports test MSE.

TopoU-Net achieves the lowest MSE on four of five datasets and remains close to
the best baseline on FashionMNIST. This is consistent with the support-ratio
analysis: for the $28\times28$ grid construction in
Table~\ref{tab:compression-profile}, the bottleneck support ratio is
$\rho_{\mathrm{bot}}=0.93$, so the bottleneck retains nearly the same number of
cells as the input rank.

\begin{table}[h]
\centering
\caption{Test MSE on image reconstruction benchmarks.}
\label{tab:reconstruction_mse}
\resizebox{1\textwidth}{!}{%
\begin{tabular}{lccccc}
\hline
\textbf{Model} & \textbf{CIFAR10} & \textbf{EMNIST} 
& \textbf{FashionMNIST} & \textbf{MNIST} & \textbf{SVHN} \\
\hline
AttentionUNet & $2.90\times10^{-3}$ & $1.72\times10^{-3}$ 
& $5.96\times10^{-3}$ & $1.90\times10^{-3}$ & $8.04\times10^{-4}$ \\
ResUNet & $1.77\times10^{-3}$ & $1.42\times10^{-3}$ 
& $\mathbf{3.95\times10^{-3}}$ & $1.58\times10^{-3}$ & $4.82\times10^{-4}$ \\
UNet & $2.46\times10^{-3}$ & $1.46\times10^{-3}$ & $5.34\times10^{-3}$ 
& $1.63\times10^{-3}$ & $7.43\times10^{-4}$ \\
UNetPP & $2.06\times10^{-3}$ & $1.53\times10^{-3}$ & $6.07\times10^{-3}$ 
& $1.53\times10^{-3}$ & $5.65\times10^{-4}$ \\
\hline
Topo-UNet & $\mathbf{1.15\times10^{-3}}$ & $\mathbf{1.01\times10^{-3}}$ 
& $4.27\times10^{-3}$ & $\mathbf{9.46\times10^{-4}}$ 
& $\mathbf{2.63\times10^{-4}}$ \\
\hline
\end{tabular}}
\end{table}

\subsection{Full compression profile analysis}

Table~\ref{tab:full_compression_profile} extends
Table~\ref{tab:compression-profile} to all node-classification datasets. The
reported path is the encoder path; the decoder reverses it.

\begin{table*}[h]
\centering
\caption{Complete compression profiles across all node classification 
datasets.}
\label{tab:full_compression_profile}
\resizebox{\textwidth}{!}{%
\begin{tabular}{lccccccccl}
\hline
\textbf{Dataset} & $n_0$ & $n_1$ & $n_2$ & $n_3$ & $\rho_0$ & $\rho_1$ 
& $\rho_2$ & $\rho_{\text{total}}$ & \textbf{Path} \\
\hline
\multicolumn{10}{l}{\emph{Heterophilic}} \\
Texas & 183 & 325 & 52 & 1 & 1.78 & 0.16 & 0.02 & 0.006 
& $0\!\to\!1\!\to\!2\!\to\!3\!\to\!2\!\to\!1\!\to\!0$ \\
Wisconsin & 251 & 515 & 89 & 1 & 2.05 & 0.17 & 0.01 & 0.004 
& $0\!\to\!1\!\to\!2\!\to\!3\!\to\!2\!\to\!1\!\to\!0$ \\
Cornell & 183 & 298 & 47 & 1 & 1.63 & 0.16 & 0.02 & 0.005 
& $0\!\to\!1\!\to\!2\!\to\!3\!\to\!2\!\to\!1\!\to\!0$ \\
Actor & 7600 & 33544 & 1188 & 1 & 4.41 & 0.04 & 0.001 & 0.0001 
& $0\!\to\!1\!\to\!2\!\to\!3\!\to\!2\!\to\!1\!\to\!0$ \\
Chameleon & 2277 & 36101 & 1507 & 1 & 15.9 & 0.04 & 0.001 & 0.0004 
& $0\!\to\!1\!\to\!2\!\to\!3\!\to\!2\!\to\!1\!\to\!0$ \\
Squirrel & 5201 & 217073 & 2346 & 1 & 41.7 & 0.01 & 0.0004 & 0.0002 
& $0\!\to\!1\!\to\!2\!\to\!3\!\to\!2\!\to\!1\!\to\!0$ \\
\hline
\multicolumn{10}{l}{\emph{Homophilic}} \\
Computers & 13752 & 245861 & 8943 & --- & 17.9 & 0.04 & --- & 0.65 
& $0\!\to\!1\!\to\!0$ \\
Photo & 7650 & 119081 & 4521 & --- & 15.6 & 0.04 & --- & 0.59 
& $0\!\to\!1\!\to\!0$ \\
WikiCS & 11701 & 216123 & 9834 & --- & 18.5 & 0.05 & --- & 0.84 
& $0\!\to\!1\!\to\!0$ \\
\hline
\end{tabular}}
\end{table*}

\subsection{Hyperparameter settings}

All TopoU-Net models use hidden dimension $d_h = 64$ for encoder/decoder 
blocks, learning rate $\eta = 0.001$ with weight decay $10^{-5}$, and 
dropout rate 0.3. Graph classification uses batch size 32, node classification 
uses full-batch training. For heterophilic graphs, the rank-3 cell is a single global cell incident to
all rank-2 triangle cells. Its feature is initialized by mean pooling over the
triangle features. Baseline models (GCN, GAT, GIN, GraphSAGE, H2GCN, MixHop, DiffPool, Graph 
U-Net) use hyperparameters from original implementations, adapted to match 
our train/validation/test splits and normalization.

\end{document}